\documentclass[10pt,letterpaper]{article}
\usepackage[top=0.85in,left=2.75in,footskip=0.75in]{geometry}

\usepackage{amsmath,amssymb}

\usepackage{changepage}

\usepackage[utf8x]{inputenc}

\usepackage{textcomp,marvosym}

\usepackage{cite}

\usepackage{nameref,hyperref}

\usepackage{microtype}
\DisableLigatures[f]{encoding = *, family = * }

\usepackage[table]{xcolor}

\usepackage{array}

\newcolumntype{+}{!{\vrule width 2pt}}

\newlength\savedwidth

\raggedright
\usepackage[aboveskip=1pt,labelfont=bf,labelsep=period,justification=raggedright,singlelinecheck=off]{caption}

\makeatletter
\renewcommand{\@biblabel}[1]{\quad#1.}
\makeatother

\usepackage{lastpage,fancyhdr,graphicx}
\usepackage{epstopdf}
\fancyheadoffset[L]{2.25in}
\fancyfootoffset[L]{2.25in}
\usepackage{amsmath}
\usepackage{amssymb}
\usepackage{amsthm}
\usepackage{bbm}
\usepackage[table]{xcolor}
\usepackage{graphicx}
\graphicspath{ {./figures/} }
\usepackage{algorithm}
\usepackage{algpseudocode}
\usepackage{subcaption}
\usepackage{pgfplots}
\usepackage{booktabs}       

\newcommand{\stim}{\mathsf{x}}
\newcommand{\out}{\mathsf{y}}
\newcommand{\conn}{\mathsf{w}}
\newcommand{\pred}{\mathsf{a}}
\newcommand{\ent}{\mathcal{H}}
\newcommand{\ntp}{n_\mathrm{TP}}
\newcommand{\ntn}{n_\mathrm{TN}}
\newcommand{\nfn}{n_\mathrm{FN}}
\newcommand{\nfp}{n_\mathrm{FP}}
\DeclareMathOperator*{\argmax}{arg\,max}
\newtheorem{lemma}{Lemma}
\definecolor{tp_blue}{RGB}{31, 119, 180}

\begin{document}
\vspace*{0.2in}

\begin{flushleft}
{\Large
\textbf\newline{Online neural connectivity estimation with ensemble stimulation} 
}
\newline
\\
Anne Draelos\textsuperscript{1*},
Eva A. Naumann\textsuperscript{2},
John M. Pearson\textsuperscript{1,2,3}
\\
\bigskip
\textbf{1} Department of Biostatistics \& Bioinformatics, Duke University, Durham, NC, USA
\\
\textbf{2} Department of Neurobiology, Duke University, Durham, NC, USA
\\
\textbf{3} Department of Electrical \& Computer Engineering, Duke University, Durham, NC, USA
\\
\bigskip

%
%





* anne.draelos@duke.edu

\end{flushleft}

\section*{Abstract}
One of the primary goals of systems neuroscience is to relate the structure of neural circuits to their function, yet patterns of connectivity are difficult to establish when recording from large populations in behaving organisms. Many previous approaches have attempted to estimate functional connectivity between neurons using statistical modeling of observational data, but these approaches rely heavily on parametric assumptions and are purely correlational. Recently, however, holographic photostimulation techniques have made it possible to precisely target selected ensembles of neurons, offering the possibility of establishing direct causal links. Here, we propose a method based on noisy group testing that drastically increases the efficiency of this process in sparse networks. By stimulating small ensembles of neurons, we show that it is possible to recover binarized network connectivity with a number of tests that grows only logarithmically with population size under minimal statistical assumptions. Moreover, we prove that our approach, which reduces to an efficiently solvable convex optimization problem, can be related to Variational Bayesian inference on the binary connection weights, and we derive rigorous bounds on the posterior marginals. This allows us to extend our method to the streaming setting, where continuously updated posteriors allow for optional stopping, and we demonstrate the feasibility of inferring connectivity for networks of up to tens of thousands of neurons online. Finally, we show how our work can be theoretically linked to compressed sensing approaches, and compare results for connectivity inference in different settings.



\section*{Introduction}
A long-standing problem in systems neuroscience is that of inferring the functional network structure of a population of neurons from its neural activity. That is, given a set of neural recordings, we would like to know which neurons influence which others in the system without \emph{a priori} knowledge of their anatomical connectivity. This problem is made difficult in two ways: First, new techniques in microscopy and neural probe technology have dramatically increased the size of recorded neural populations \cite{naumann2016whole, lu2017video,Stringer2019Spont, steinmetz2020neuropixels}, posing a computational challenge. Second, the fact that typical interventions in these systems remain broad and non-specific poses problems for causal inference \cite{mehler2018lure,sadeh2020patterned, sadeh2020theory}.

However, recent advances in precision optics and opsin engineering have resulted in photostimulation tools capable of precisely targeting individual neurons and neuronal ensembles \cite{emiliani2015all,packer2015simultaneous,pegard2017three,yang2018holographic,marshel2019cortical}. This suggests that a combination of simultaneous recording and  \emph{selective} stimulation could potentially allow for functional dissection of large-scale neural circuits. Yet the most common methods for inferring functional connectivity are purely statistical models, applied to observational data \cite{paninski2004maximum, pillow2008spatio}. They do not consider causal inferences based on interventions (though cf. \cite{hu2009reconstruction, shababo2013bayesian, aitchison2017model}), and often make stringent parametric assumptions, which can limit their ability to recover connectivity even in simulations \cite{lutcke2013inference, das2020systematic}. 

Here, we take a different approach to inferring functional connectivity based on targeted stimulation of small, randomly-chosen neural ensembles. We adopt the framework of group testing \cite{dorfman, du2000combinatorial, aldridge2019group}, an experimental design strategy that relies on simultaneous tests of multiple items. Group testing reduces the complexity of detecting rare defects (here, true connections) from linear to logarithmic in the number of units, allowing it to scale to large neural populations. We show that this approach, which makes only mild statistical assumptions, can be significantly more efficient than testing single neurons in isolation. Furthermore, we propose an efficient convex relaxation of the inference problem that is related to marginal Bayesian posteriors for the existence of individual connections. Finally, we show that an optimization scheme based on dual decomposition offers a highly parallelizable, GPU-friendly problem formulation that allows us to perform inference on a population of $10^4$ neurons in the online setting. Taken together, these ideas suggest new algorithmic possibilities for the adaptive, online dissection of large-scale neural circuits.

\section*{Methods}
\subsection*{Network inference as group testing}
Our goal is to recast the problem of inferring \emph{functional} connectivity between neurons as a group testing problem. This functional connectivity has only to do with the ability of one neuron to cause changes  in the activity of another and does not imply a direct synaptic connection. Thus, two neurons may be functionally connected when no direct synaptic connection exists. In particular, we are not addressing the problem of unobserved confounders---unrecorded neurons that mediate observed interactions. Nonetheless, functional connectivity remains a quantity of intense interest, since it is likely to reflect patterns of influence and information flow in neural circuits \cite{feldt2011dissecting, Grosenick2015closed}. 

To establish conventions, it will help to consider a simple baseline protocol for establishing functional connectivity: let each test consist of stimulating a single neuron, with the test possibly repeated several times. In this setup, a stimulated neuron $i$ can be considered functionally upstream of a second neuron $j$ if $j$ typically alters its activity in response to stimulation of $i$. More precisely, we assume that there exists a test $h: \mathcal{D} \rightarrow \lbrace 0, 1 \rbrace$ that concludes from data whether stimulation of $i$ altered activity in $j$. This approach has two important advantages: First, we do not need to assume that excitation of $i$ results in excitation of $j$, only that the test detects a difference. In other words, we are not limited to excitatory connections. Second, while a given test might make parametric assumptions about the data, our subsequent analysis will be agnostic to these assumptions. Thus the ability to consider a multiplicity of tests offers us a degree of statistical flexibility not present in approaches that must rely on, e.g., linearity of synaptic contributions from different neurons. But these benefits imply a tradeoff: we will only be able to amass statistical evidence for the \emph{existence} of such connections, and possibly their signs, but not their relative strength. We view this as a reasonable tradeoff in cases where the structure of connections is of primary concern, with the added observation that, once connections are identified, a second round of more focused testing or post-hoc methods can serve to establish strengths.

To model the effects of ensemble photostimulation, we assume that all neurons in the target set receive roughly the same light intensity, and that this intensity is sufficient to evoke a detectable response if any one of the neurons is connected to some other. Moreover, we assume that stimulation is strong enough that even, in cells receiving mixed excitatory and inhibitory connections, one will predominate. That is, given $N$ observed neurons subjected to stimulations indexed by $t$, let $\stim_{tj} = 1$ if neuron $j$ is stimulated on round $t$, and $\conn_{i\rightarrow j} = 1$ if neuron $i$ functionally influences neuron $j$. With these conventions, we define the predicted activation of unit $i$ as the logical OR of all the connections 

\begin{equation}
    \label{defn:pred}
    \pred_{ti}(\conn) = \bigvee_{j=1}^N \conn_{ij} \stim_{tj} = \max(\mathbf{\conn}_{i\cdot} \odot \mathbf{\stim}_{t\cdot})
\end{equation}
and the outcome of the hypothesis test $h$ with false positive rate $\alpha$ and false negative rate $\beta$ as 
\begin{align}
    \label{defn:stat_model}
    \out_{ti}|(\pred_{ti}=1) \sim \mathrm{Bern}(1-\beta) && \out_{ti}|(\pred_{ti}=0) \sim \mathrm{Bern}(\alpha) \, .
\end{align}
Note that this assumes $\pred$ is a sufficient statistic for the outcome $\out$, which may not hold if, e.g., false positive rates increase with the number of stimulated neurons \cite{aldridge2019group}.

This formulation, in which multiple units are combined into a single test that returns a positive result if any of the individual units would alone, is known as the group testing problem. Originally devised by Dorfman \cite{dorfman} as an efficient means of testing for syphilis in soldiers, group testing has spawned an enormous literature, with applications in medicine, communications, and manufacturing (recently reviewed in \cite{aldridge2019group}). As shown by Atia and Saligrama \cite{atia2012boolean}, this can be cast in the language of information theory as a channel coding problem with $\stim$ the codebook and $\out$ the channel output. Moreover, \cite{atia2012boolean} demonstrated that when $\stim$ is a randomized testing strategy to find $K$ true positives, the number of tests required to solve the problem with exponentially small average-case error is both upper and lower bounded asymptotically by $K \log N$, even when tests are noisy and $K \sim {o}(N)$. 

The problem we consider here is more specifically one of noisy group testing in the sparse regime. That is, we allow the test to be corrupted as specified in (\ref{defn:stat_model}) and assume $K \sim \mathcal{O}(N^\theta)$ with $\theta \in (0, 1)$. Within this regime, approaches principally differ along two axes: adaptive versus non-adaptive test designs and the method used to infer $\conn$. In non-adaptive designs, the tests are fixed in advance, allowing them to be run in parallel at the cost of some statistical efficiency (though not necessarily asymptotically \cite{aldridge2012adaptive, chan2014non}). Adaptive designs, by contrast, are chosen sequentially, often to optimize the information gained with each test. Below, we consider both methods, but for the remainder of this section and the next, we focus on the second axis: the method of inferring $\conn$. 

For simplicity we focus on a single output neuron $j$ and its potential incoming connections $\conn_{ij}$ (see Fig.~\ref{fig:fig1}). Since the inference problems for $\conn_{ij}$ and $\conn_{ij'}$ are completely independent for $j'\neq j$, these problems can be trivially parallelized, and we drop the index $j$ in what follows. Given (\ref{defn:pred}) and (\ref{defn:stat_model}), we can infer the true connections by maximizing the total log likelihood over all $T$ tests:
\begin{align}
    \label{bin_LL}
    \log p(\lbrace \out_t \rbrace | \lbrace \conn_i, \stim_t\rbrace)
    &= T \log(1-\alpha) - \log \frac{1-\alpha}{\alpha} \sum_t \out_t  \\
    &\phantom{=}
    - \log \frac{1-\alpha}{\beta}\sum_t \pred_t(\conn)  + \log \frac{(1-\alpha)(1-\beta)}{\alpha\beta}\sum_t \out_t \pred_t(\conn)  \nonumber \\
    &= \sum_t \left[ \log \frac{(1-\alpha)(1-\beta)}{\alpha\beta} \out_t - \log \frac{1-\alpha}{\beta} \right] \pred_t(\conn) + \text{const} \nonumber\;, 
\end{align}
where the constant does not depend on $\conn$. For any reasonable test, we expect ${1-\beta} > {\alpha}$ (i.e., the true positive rate exceeds the false positive rate) and $1-\alpha > \beta$ (true negative rate exceeds false negative rate), so that the term in brackets is positive when $\out_t = 1$ and negative when $\out_t = 0$. Thus the maximum likelihood solution is one in which the bits $\pred_t(\conn)$ and $\out_t$ most often match, similar to one-bit compressed sensing \cite{boufounos20081, malioutov2012boolean}. 

\begin{figure}[!h]
    \captionsetup[subfigure]{position=top, labelfont=bf}
    \centering
    \begin{subfigure}[c]{0.4\textwidth}
        \subcaption{}
        \centering
        \includegraphics[width=\textwidth]{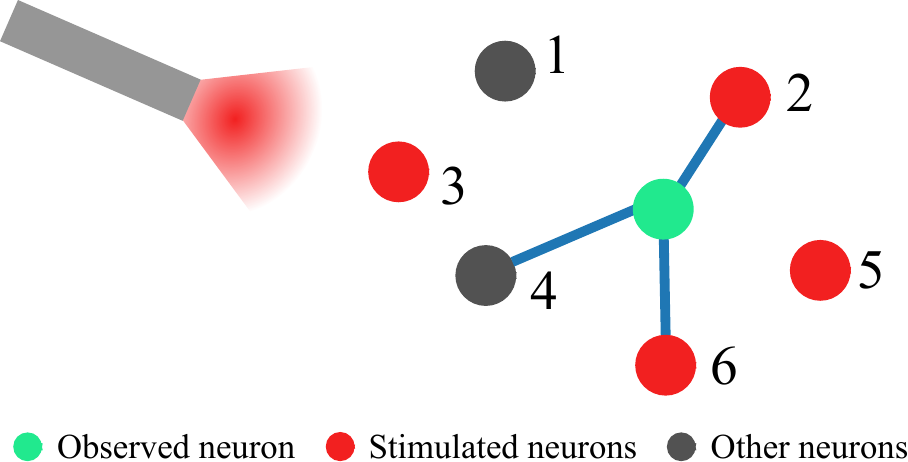}
    \end{subfigure}
    \begin{subfigure}[c]{0.4\textwidth}
        \centering
        \begin{tabular}{cccccccc}  
            & \multicolumn{6}{c}{Neuron} \\
            Test  & 1 & {\bf \color{tp_blue} 2} & 3 & {\bf \color{tp_blue} 4} & 5 & {\bf \color{tp_blue} 6} & Result\\
            \midrule
            1 & 0 & 0 & 1 & 1 & 0 & 1 & 1      \\
            \rowcolor{gray!20} 2 & 0 & 1 & 1 & 0 & 1 & 1 & 1      \\
            3 & 1 & 0 & 1 & 0 & 1 & 0 & 0      \\
            4 & 1 & 1 & 0 & 0 & 0 & 0 & 1      \\
            5 & 0 & 0 & 1 & 0 & 1 & 0 & 0      \\
            \bottomrule
        \end{tabular}
    \end{subfigure}
    \linespread{1.0}\selectfont{\caption{\textbf{Neural stimulation as group testing. (a)} Holographic photostimulation targets specific subsets of neurons (red), which result in activity in a target neuron (green). Neurons 2, 4, and 6 are functionally connected to the target neuron (blue lines), and stimulation of any one of them is sufficient to evoke activity. {\bf (b)} Equivalent group testing matrix ($\stim_{ti}$), with each row a test and each column a neuron. The result of each test ($\out_t$) is a logical OR of the stimulation variables for the true connections (blue). Test 2 (gray) corresponds to the stimulation in (a).}
    \label{fig:fig1}}
\end{figure}


Unfortunately, this integer programming problem is NP-hard in general \cite{aldridge2019group}, so  approximate solution methods must be used. Previous approaches have used Monte Carlo methods like Gibbs Sampling \cite{knill1996interpretation} and message-passing approaches like Belief Propagation \cite{sejdinovic2010note, aldridge2019group}. A third class of approaches \cite{malioutov2012boolean} proposes to relax the binary variables $\conn_i \in \lbrace 0, 1\rbrace \rightarrow w_i \in [0, 1]$ and solve a linear program to minimize $\sum_i w_i + \gamma \sum_t \xi_t$, with the $\xi_t$ slack variables representing noise (bit flips between $\out_t$ and $\pred_t$) and $\gamma$ parameterizing the sparsity of the solution. This method indeed performs well in practice \cite{malioutov2012boolean, aldridge2019group} and makes no assumptions about the form of the noise, though it does require tuning $\gamma$, which may be difficult when the true number of defects $K$ is unknown.

Here, we propose an alternate relaxation based on independently relaxing the variables $\pred_t \rightarrow a_t$ and relating these to the $w_i$ via constraints. That is, instead of the $N + T$ variables $\lbrace w_i, \xi_t\rbrace$, we will choose to optimize over $\lbrace w_i, a_t\rbrace$, solving 
\begin{equation}
    \label{relaxed_LP}
    \max_{\lbrace w_i, a_t \rbrace} \; \sum_t c_t a_t \quad
    \text{ subject to} \quad \stim_{ti} w_i \le a_t \le \sum_i \stim_{ti} w_i, \quad w_i, a_t \in [0, 1] \; 
\end{equation}
with $c_t = \log \frac{(1-\alpha)(1-\beta)}{\alpha\beta} \out_t - \log \frac{1-\alpha}{\beta}$. The constraints we impose on the new variables $a_t$ can be understood from (\ref{defn:pred}) by noting that the maximum of a set of positive variables must be greater than or equal to each of them individually and at most equal to their sum. As we will show, this linear program in $N + T$ variables with $\sum_{ti} \stim_{ti} + T$ constraints may be large (and grows with the number of tests) but can nonetheless be solved efficiently even for sizable neural populations. Unfortunately, there is no guarantee that the solution to (\ref{relaxed_LP}) produces a solution to the original integer optimization problem, and one is left with the problem of finding some method of rounding $w_i$ to produce a binary solution \cite{schrijver1998theory}. Fortunately, as we will argue below, this is unnecessary, and a slight alteration to (\ref{relaxed_LP}) gives the $w_i$ an attractive interpretation.

\subsection*{Relaxed group testing as Bayesian inference}
In the discussion above we focused on maximum likelihood decoding, since this procedure has exponentially small error in the large $T$ limit \cite{atia2012boolean, aldridge2012adaptive,baldassini2013capacity, chan2014non}. However much of this work also assumes that the number of true positives $K$ is known. In our case, by contrast, we might only have weak beliefs about the distribution of connections across neurons. Moreover, with a fixed time budget for data collection, we would benefit from the option to either stop the experiment early (if all connections have been found) or produce an estimate of uncertainty for the $\conn_{i}$ at the end of the experiment.

Thus we consider the problem of Bayesian inference for the likelihood given in (\ref{bin_LL}) with Bernoulli priors $\conn_i \sim \mathrm{Bern}(\pi_i)$. In this case, the log posterior takes the form 
\begin{equation}
    \label{w_posterior}
    \log p(\conn \vert \stim, \out ) = \sum_t c_t \pred_t(\conn) + \sum_i \mu_i \conn_i - \log \mathcal{Z}\; ,
\end{equation}
with $\mu_i = \log \frac{\pi_i}{1-\pi_i}$ and $\mathcal{Z}$ a normalizing constant. Clearly, the posterior is in exponential family form, with sufficient statistics $\conn_i$ and $\pred_t(\conn)$. Full inference requires computation of $\mathcal{Z}$, which is practically infeasible for $N$ or $T$ large. However, we are primarily concerned with posterior (marginal) beliefs about individual connections, so we might settle for only knowing $p(\conn_i\vert \stim, \out )$.

Luckily, two facts already mentioned allow us to compute these marginals efficiently: First, (\ref{w_posterior}) is in exponential family form, and second, the $\conn_i$ are sufficient statistics for the posterior. Taking a Variational Bayes approach \cite{blei2017variational}, we rewrite inference as an optimization problem. Let
\begin{equation}
    \label{generic_elbo}
    q_*(\conn ) \equiv \argmax_{q(\conn) \in \mathcal{Q}}\; \mathbb{E}_q[\log p(\out \vert \conn, \stim) + \log p(\conn)] + \ent[q(\conn)] \;,
\end{equation}
where $\mathcal{Q}$ is some class of distributions over which we optimize and $\ent = \mathbb{E}_q[-\log q(\conn)]$ is the entropy. This is equivalent \cite{blei2017variational} to minimizing the KL divergence between $q_*(\conn)$ and $p(\conn \vert \stim, \out)$, with $D_{KL}(q_*\Vert p) = 0$ if and only if $q_* = p$ almost everywhere.

We exploit the fact that we know the form of the posterior to choose a class $\mathcal{Q}$ that contains $p(\conn|\stim, \out)$, since this will imply that (\ref{generic_elbo}) yields the true posterior. The obvious choice is to take $\mathcal{Q}$ to be the exponential family defined by the sufficient statistics $\conn_i$ and $\pred_t$. However, instead of the natural parameters corresponding to these sufficient statistics, we will define them in terms of the expectations $w_i \equiv \mathbb{E}_q[\conn_i]$ and $a_t \equiv \mathbb{E}_q[\pred_t]$. In optimization language, the latter are the primal variables and the former the duals, which are related to one another through derivatives of the free energy $\log \mathcal{Z}$ \cite{wainwright2008graphical, blei2017variational}. With this choice, we can write
\begin{align}
    \label{polytope_opt}
    (\mathbb{E}[\conn_i], \mathbb{E}[\pred_t]) \equiv \argmax_{(w, a) \in \mathcal{M}}\; &\sum_t c_t a_t + \sum_i \mu_i w_i + \ent(w, a)  
\end{align}
where $\mathcal{M}$ is the marginal polytope, the set of marginals feasible under all possible distributions \cite{wainwright2008graphical}. Clearly, since $\conn_i$ and $\pred_t$ are binary, we have $\mathbb{E}[\conn_i] = P(\lbrace \conn_i = 1 \rbrace)$, $\mathbb{E}[\pred_t] = P(\cup_{j, \stim_{tj}=1} \lbrace \conn_j = 1 \rbrace)$, and the constraints in (\ref{relaxed_LP}) follow from simple containment and union bounds for any $P$. More generally, letting $\mathcal{S}_t = \lbrace j \vert \stim_{tj} = 1 \rbrace$, there are additional consistency conditions on the $a_t$:
\begin{equation}
    \label{extra_polytope}
    a_t \le P(\cup_{t' \in \mathcal{T}} \lbrace \pred_{t'} = 1\rbrace) \le \sum_{t' \in \mathcal{T}} a_{t'} \quad 
    \text{ whenever}\quad \mathcal{S}_t \subset \bigcup_{t' \in \mathcal{T}} \mathcal{S}_{t'}.
\end{equation}
That is, whenever any subset of trials $\mathcal{T}$ includes all neurons stimulated on trial $t$, $a_t$ is bounded above by the sum of the $a_{t'}$ from these other trials. 

However, if we allow $\ent$ to take values in $\mathbb{R} \cup \lbrace \infty \rbrace$, defining $\overline{\ent}(w, a) = \infty$ for $(w, a) \notin \mathcal{M}$, then we can write
\begin{align}
    \label{elbo}
    (\mathbb{E}[\conn_i], \mathbb{E}[\pred_t]) \equiv \argmax_{\lbrace w_i, a_t\rbrace}\; &\sum_t c_t a_t + \sum_i \mu_i w_i + \overline{\ent}(w, a) \\ 
    \text{s.t. }\; &\stim_{ti} w_i \le a_t \le \sum_i \stim_{ti} w_i, \quad w_i, a_t \in [0, 1] \nonumber,
\end{align}
where again, $\overline{\ent}$ incorporates the constraints in (\ref{extra_polytope}). This is equivalent to (\ref{relaxed_LP}) when we assume flat priors on $\conn_i$ ($\mu_i = 0$) and no entropy term. In other words, the relaxed $a_t$ and $w_i$ appearing in (\ref{relaxed_LP}) are approximate \emph{posterior probabilities} for the binary variables $\pred_t$ and $\conn_i$, and this relation is exact when the entropy term $\overline{\ent}$ is included as a regularizer. Thus, solving the optimization (\ref{elbo}) allows us to compute posterior marginals for the connections, even though we cannot write down $p(\conn \vert \stim, \out)$.

\subsection*{Optimization and online inference}
The above arguments show that posterior inference for group testing can be reduced to the variational problem (\ref{elbo}). However, two difficulties remain: First, calculating $\overline{\ent}(w, a)$, requires knowing the exponential family normalizing factor $\mathcal{Z}$, which is intractable in general. Second, we need an efficient method for solving (\ref{elbo}) for very large problems. Note again that we have only been considering the case of a single output neuron, which results in a convex program with $N + T$ variables and $2N + 3T + NT$ nominal constraints (\ref{relaxed_LP}). When generalized to the full network, we will have $N$ independent (and thus parallelizable) programs of this size, indicating both high memory and computational requirements. Yet, as we will show, further simplifications are possible that allow solutions to (\ref{elbo}) to be implemented even for $N > 10^4$ in the online setting.

We begin by considering a slightly more general exponential family $\widetilde{\mathcal{Q}}$ in which the $\pred_t$ as well as the $\conn_i$ are fundamental variables, with (\ref{defn:pred}) enforced by constraint:
\begin{equation}
    \label{dual_decomp}
    \log \tilde{q}_{\eta, \nu}(\conn, \pred) = \sum_t \gamma_t \pred_t + \sum_i \delta_i \conn_i - \sum_t \eta_t (\pred_t - \sum_i \stim_{ti} \conn_i) - \sum_{ti} \stim_{ti}\nu_{ti}  (\conn_i - \pred_t) - \log \mathcal{Z}(\eta, \nu)\, ,
\end{equation}
with $\nu, \eta \ge 0$. Note that this will be related to forming the Lagrangian of the problem (\ref{elbo}), but here, we are instead defining a set of probability distributions with $\sup_{\eta, \nu \ge 0}  \tilde{q} \in \mathcal{Q}' \supset \mathcal{Q}$. That is, as the constraint forces are maximized, all distributions satisfy the explicit constraints in (\ref{elbo}), though they are not guaranteed to satisfy those in (\ref{extra_polytope}). We find that, in practice, this does not affect the accuracy of recovery. 

What is important to note here is that the introduction of dual variables has effectively decoupled $\conn_i$ from $\pred_t$, since their dependency structure is a bipartite graph. Moreover, conditioned on the dual variables, the primal variables are all \emph{independent}. Following the derivation leading to (\ref{elbo}) we can now pose an equivalent optimization problem: 
\begin{align}
    \label{univariate_decomp}
    &\sup_{\substack{\eta, \nu \ge 0 \\ w_i, a_t \in [0, 1]}}\; \sum_t \mathcal{L}_t(a_t, \eta, \nu) + \sum_i \mathcal{L}_i(w_i, \eta, \nu) \, ,\\ 
    \mathcal{L}_t &= \left(c_t - \eta_t + \sum_i \stim_{ti}\nu_{ti} \right)a_t + \ent_2(a_t) \label{a_lagrangian}\\ 
    \mathcal{L}_i &= \left(\mu_i + \sum_t \stim_{ti} \eta_t - \sum_t \stim_{ti}\nu_{ti} \right)w_i+ \ent_2(w_i) \label{w_lagrangian}\, , 
\end{align}
with $\ent_2(x) = -x \log x - (1-x) \log(1-x)$ the entropy of a binary variable with mean $x$ (measured in nats). 
The univariate maximizations over $a_t$ and $w_i$ can easily be solved numerically: 
\begin{align}
    a^*_t = f\left(c_t - \eta_t + \sum_i \stim_{ti}\nu_{ti}  \right) &&
    w^*_i = f \left(\mu_i + \sum_t \stim_{ti}\eta_{t}  - \sum_t \stim_{ti} \nu_{ti} \right) \label{true_soln} \, ,
\end{align}
where $f(x) = e^x/(1+e^x)$ is the logistic function. This formulation naturally leads to a dual decomposition approach \cite{bertsekas2016} in which we first maximize exactly over $w$ and $a$ then maximize (\ref{univariate_decomp}) at the resulting optimum with respect to $\eta$ and $\nu$. Alternately, we can bound the entropy $\ent_2$ by a quadratic (Supplementary section \ref{app_H_bounds}), for which we have the solution:
\begin{align}
    a^*_t &= \left[1 - \left(\frac{1}{2}\right)^{\sum_i \stim_{ti}} + \frac{1}{\sigma} \left(c_t - \eta_t + \sum_i \stim_{ti}\nu_{ti}  \right)\right]_{[0, 1]} \label{a_soln}\\
    w^*_i &= \left[\frac{1}{2} + \frac{1}{\sigma} \left(\mu_i + \sum_t \stim_{ti}\eta_{t}  - \sum_t \stim_{ti} \nu_{ti} \right)\right]_{[0, 1]} \label{w_soln} \, ,
\end{align}
where $[\cdot]_{[0, 1]}$ indicates truncation to the unit interval and $\sigma \in (0, 4]$ is a regularization parameter. In practice, this more weakly regularized approach, which results in overconfident posteriors, performs better when binarizing $w$ to reconstruct the underlying network.

This approach is summarized in Algorithm \ref{dual_decomp_alg}. Thanks to the decoupled nature of (\ref{dual_decomp}), gradient updates for $\eta$ and $\nu$ can be performed in parallel, so efficient GPU implementations are possible. The key limitation for this approach is memory: while the $\nu$ matrix is sparse (effectively masked by $\stim$), one must still maintain space for $a$, $w$, $c$, $\mu$, $\eta$, and $\nu$ for $\mathcal{O}(NST)$ parameters, with $S$ the average number of neurons stimulated per trial. Thus, while we do benefit from using first-order methods with momentum like Adam \cite{kingma2014adam}, these also come at the additional memory cost of $\mathcal{O}(2NST)$ running mean and variance estimates, making it impractical for systems larger than $\sim10^3$ neurons.
\begin{algorithm}
\caption{Dual decomposition inference}
\label{dual_decomp_alg}
\begin{algorithmic}[1]
\State {\bf Initialize:} $\eta_t, \nu_{ti} \gets 0$ 
\State
\While{not converged}
\State Solve for $a^*_t$, $w^*_i$ via (\ref{true_soln}) or (\ref{a_soln}), (\ref{w_soln})
\State $\eta_{t} \gets \eta_t - \alpha (\sum_i \stim_{ti} w^*_i - a^*_t)$
\State $\nu_{ti} \gets \nu_{ti} + \alpha \stim_{ti} (w^*_i - a^*_t)$
\EndWhile
\end{algorithmic}
\end{algorithm}

Along different lines, we can further reduce memory requirements for very large systems by simply limiting the gradient updates in Algorithm \ref{dual_decomp_alg} to the $\eta_t$ and $\nu_{ti}$ for the most recent $\tau$ time steps. That is, for $\tau = 50$, we stop updating $\eta_2$ for $t > 52$. This halts the memory growth of the algorithm with number of tests performed, for a space complexity of $\mathcal{O}(NS\tau + 2N^2)$. As we will demonstrate in the next section, this allows us to perform inference on a network of $10^4$ neurons (one hundred million potential connections) using gradient descent with negligible loss of accuracy. In fact, our GPU implementation using CuPy \cite{cupy_learningsys2017} performed each gradient descent iteration in under 2 seconds, fast enough to perform online inference during experiments.

Finally, we note that our identification of the $w_i$ with the posterior $p(\conn_i|\stim, \out)$ naturally lends itself to adaptive testing. In typical adaptive algorithms, one is interested in maximizing some expected information gain or minimizing uncertainty, which can pose difficult computational problems when only point estimates are available \cite{du2000combinatorial, aldridge2019group}. Here, however, we can trivially select those units with greatest posterior uncertainty for priority testing. In a different vein, access to calibrated uncertainties also facilitates either early stopping (when a minimum certainty is required) or optimal test allocation (when the number of tests is limited).

\subsection*{Interpolation between group testing and compressed sensing}
As mentioned above, our approach possesses some similarities with 1-bit compressed sensing \cite{boufounos20081, malioutov2012boolean}, which has also been studied as a potential method for inferring functional connectivity in neural networks \cite{hu2009reconstruction, fletcher2011neural}. Here, we show that our relaxation (\ref{relaxed_LP}) can make the link between the two precise.

Consider a version of the one-bit compressed sensing problem in which we assume continuous weights $w_i$ while retaining binary stimulations $\stim_t$ and outcomes $\out_t$. Then (\ref{defn:stat_model}) is replaced by 
\begin{equation}
    \label{1bCS}
    \log p(\lbrace \out_t \rbrace | \lbrace w_i, \stim_t\rbrace) = \sum_t c_t a_t + \text{const} \; ,
\end{equation}
where we recall that $c_t$ depends on $\out_t$ and now the predictor $a_t \equiv \sum_i w_i \stim_{ti}$. Of course, maximizing this likelihood is equivalent to performing logistic regression, and if we employ a $\mu_i$ regularization as in (\ref{w_posterior}), this is equivalent to a LASSO problem. Here, we will show that (\ref{1bCS}) and (\ref{defn:stat_model}) can be encompassed in a single formulation that naturally interpolates between the two.

The key to this is to note that in our original formulation, when we relax $\pred_t$ and $\conn_i$ to $a_t$ and $w_i$, $a_t$ becomes an independent variable constrained by the values of $w_i$:
\begin{equation}
    w_i \stim_{ti} \le a_t \le \sum_i w_i \stim_{ti}
\end{equation}
which we can rewrite in the suggestive form
\begin{equation}
    \lVert \mathbf{w \odot \stim}_t \rVert_\infty \le a_t \le \lVert \mathbf{w \odot \stim}_t \rVert_1 
\end{equation}
since all our $w_i \ge 0$.
That is, $a_t$ becomes equal to $\lVert \mathbf{w \odot \stim}_t \rVert_p$ for some $p \ge 1$. When this reaches the lower bound, we have group testing, while the upper bound represents one-bit compressed sensing as in (\ref{1bCS}). 

\section*{Results}
We tested the performance of Algorithm \ref{dual_decomp_alg} in both the offline (all data) and online (one test at a time) settings. In the offline setting, we considered Bernoulli designs in which each neuron was stimulated independently on each trial with probability $p_\mathrm{stim} = S/N$. In the online setting, we considered both Bernoulli designs and adaptive designs, in which the top $S$ most uncertain neurons (those with $w_i$ closest to $\frac{1}{2}$) were selected for the next test. We used randomly generated binary graphs $\conn_{ij}$ in which each link appeared independently with probability $K/N$. 

We also distinguish two separate problems: uncertainty quantification and recovery. The former focuses on efficient calculation of accurate Bayesian posteriors using the formulation (\ref{elbo}), while the latter focuses on binarizing $w$ to produce the most likely underlying $\conn$. Thus, for uncertainty we use the tigher entropy bound $\ent_2$ and priors defined by $\mu$, while for recovery we use the computationally cheaper quadratic approximation to $\ent$ with $\sigma \ll 1$, $\mu = 0$ and a classification threshold at $w = \frac{1}{2}$. In our experiments, this weak regularization, which resulted in overconfident posteriors, consistently produced better recovery. The experiments presented here focus on the recovery problem. 

Unless otherwise stated, we use a {\bf base case} of $N=1000$, $K = N^{0.3} \approx 8$ incoming connections per neuron, $S = 10$ stimulated neurons per test, $\alpha=\beta=0.05$, $\mu=0$, $\sigma=0.1$, and Adam \cite{kingma2014adam} with step size $0.01$, $\beta_1=0.9$, and $\beta_2=0.999$ for optimization in the offline setting, with convergence typically achieved within 50 steps.
All experimental simulations were run on a 2018 custom-built desktop machine with 128 GB of system memory, a 14 core 3.1 GHz Intel i9-7940X processor, an NVIDIA Titan Xp GPU with 12 GB of memory, and running Ubuntu 18.04.4 LTS.

\subsection*{Network recovery in the offline setting}
Fig.~\ref{fig:fig2} demonstrates the effectiveness of our algorithm in correctly recovering a binary network. The inferred system is initially regularized toward the maximum entropy solution at ($w=\frac{1}{2}$), but as the number of tests increases, connections are rapidly segregated toward 0 and 1, with classification based on a threshold at 0.5. True negatives are learned quickly at the expense of incorrectly classifying some true positives (drop in sensitivity as specificity rises), but the algorithm eventually corrects for this behavior (Fig.~\ref{fig:fig2}a). Tests with higher error rates show decreased performance (Fig.~\ref{fig:fig2}b), but this is mitigated at larger numbers of tests. Finally, in comparison with a naive model that stimulates single neurons ($S=1$, Supplemental section 2) group testing dominates on both measures after about 500 trials (Fig.~\ref{fig:fig2}c,d).

\begin{figure}[!h]
    \includegraphics[width=\textwidth]{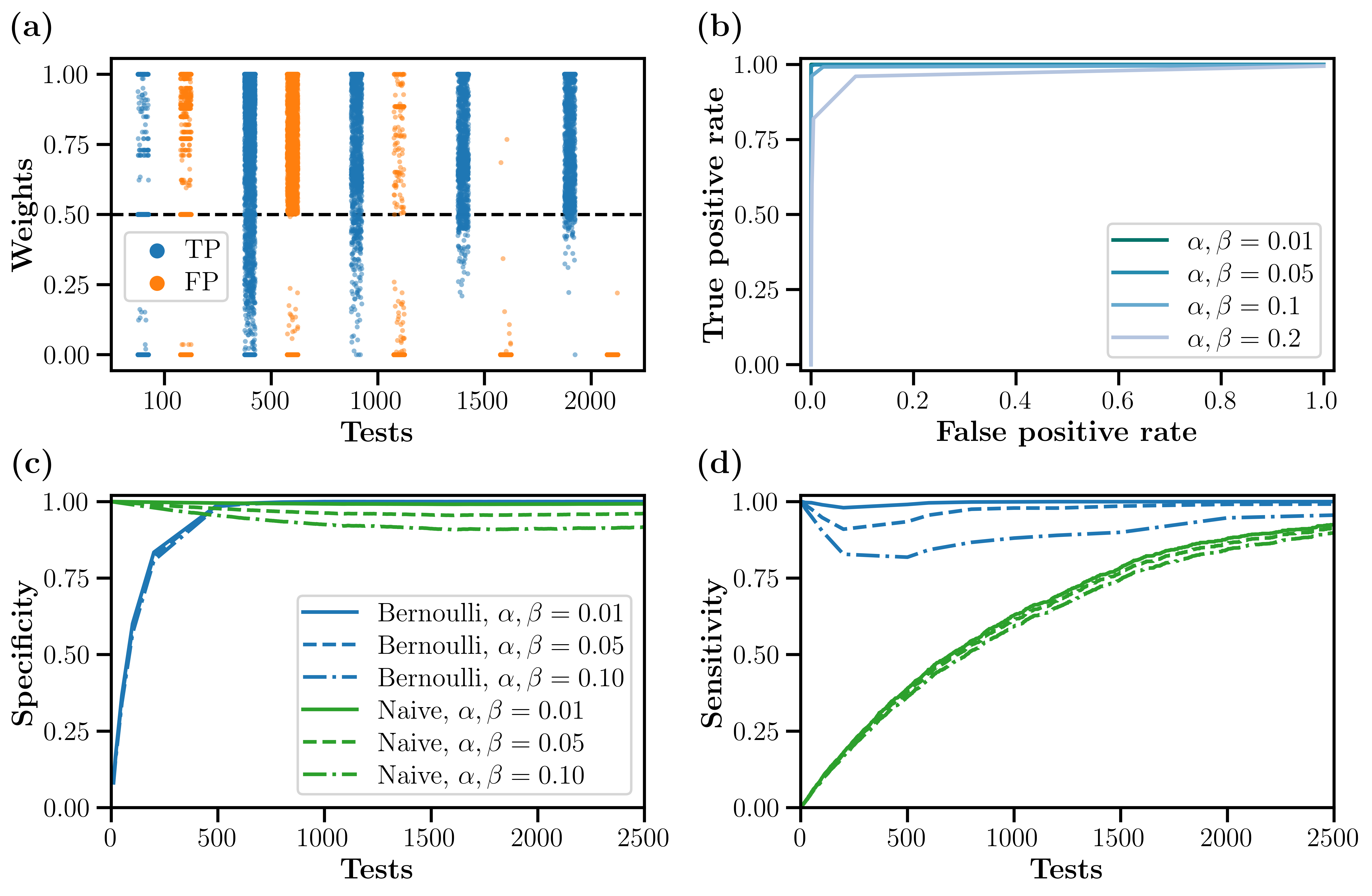}
    \linespread{1.0}\selectfont{\caption{\textbf{Offline network recovery performance. (a)} Recovery improves with increasing numbers of tests. Dots (jittered for visibility) indicate posterior estimates for true connections (blue) and spurious connections (orange) as tests are added. The classification threshold is at 0.5 (dotted line), and we do not plot the nearly $10^6$ true negatives at 0.
    {\bf (b)} ROC curves as a function of test error rates. Even as $\alpha$ and $\beta$ grow, performance degrades only moderately. 
    {\bf (c, d)} Specificity and sensitivity as a function of test number and error rate. The naive approach gradually identifies positive connections, while group testing quickly separates positive and non-connections across the 0.5 threshold. 
    }
    \label{fig:fig2}}
\end{figure}

Figure \ref{fig:suppfigseeds} shows the variation due to setting different random seeds, along with the time per iteration. Each set of results (specificity and sensitivity for all tests) takes about 20 minutes in total to run when using 50 iterations for batch fitting. To run 500 tests for a N=1000 system, for example, would only take up to 3.5 minutes (see specific timing information for each set of tests in Fig. \ref{fig:suppfigseeds}c).

\begin{figure}[!h]
    \includegraphics[width=\textwidth]{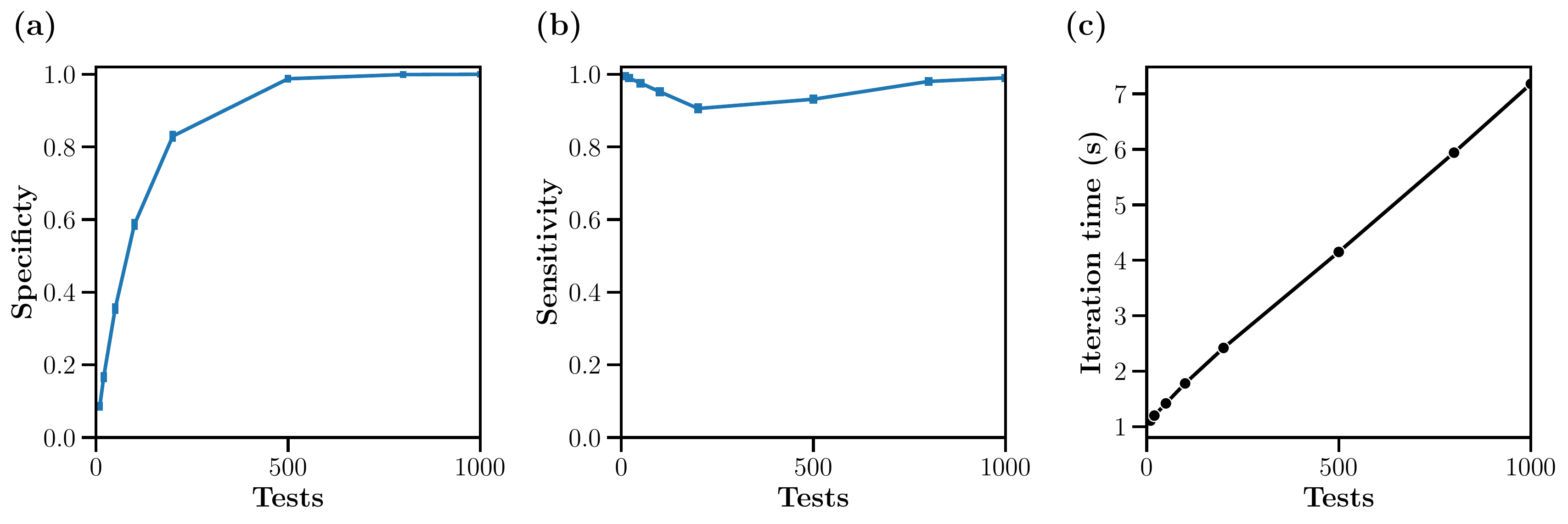}
    \linespread{1.0}\selectfont{\caption{\textbf{Variability and timing. (a, b)} Specificity and sensitivity, respectively, for the base case run with different random seeds (n=20, CI=95\%)
    \textbf{(c)} Time per iteration in seconds, averaged across 50 iterations, as a function of the number of tests for the base case. 
    }
    \label{fig:suppfigseeds}}
\end{figure}

We additionally tested our method on networks with denser sets of connections, $K=N^{\theta}$ where $\theta = [0.3,0.4,0.5]$. As Figure \ref{fig:suppfigsparse} shows, this method is robust to the number of connections per neuron. As the network becomes less sparse, the specificity and sensitivity decrease, but only slightly. 

\begin{figure}[!h]
    \includegraphics[width=\textwidth]{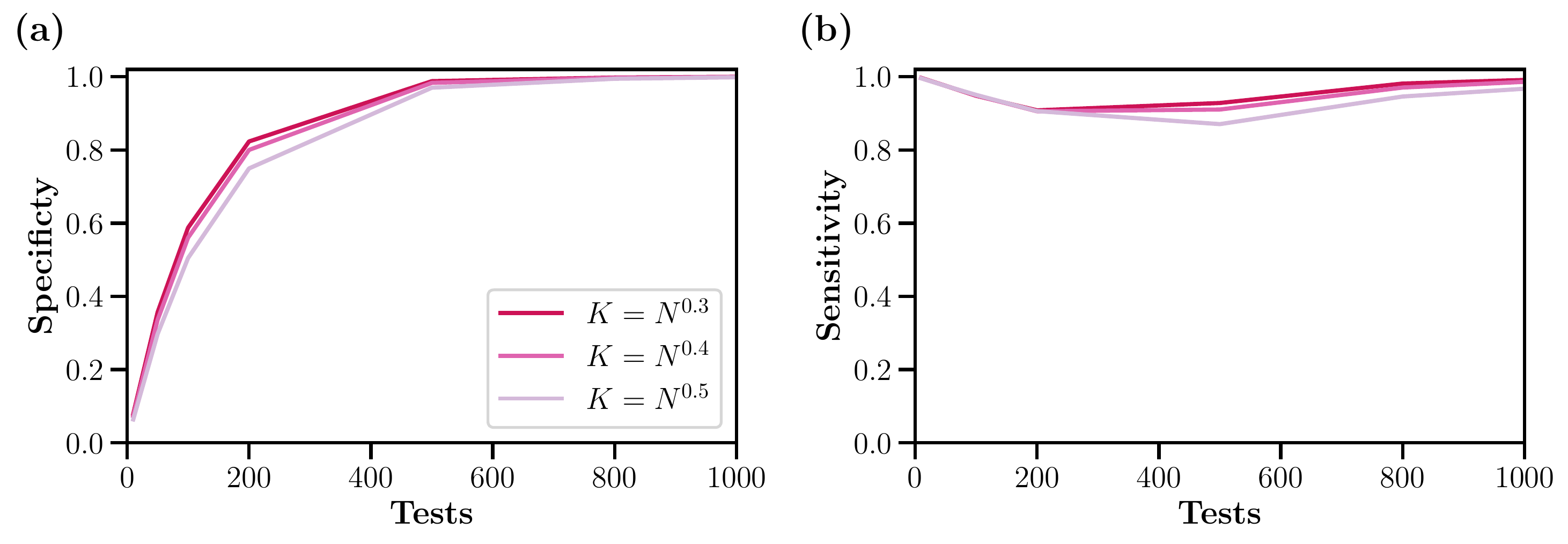}
    \linespread{1.0}\selectfont{\caption{\textbf{Sparsity. (a, b)} Specificity and sensitivity, respectively, for different levels of network sparsity. 
    }
    \label{fig:suppfigsparse}}
\end{figure}

In the base case, we used $S=10$ as the size of our stimulation group. Figure \ref{fig:suppfigNS} shows the effect of varying this stimulation group size. As $S$ grows, the efficiency of group testing increases. Indeed, the optimal choice for $S$ is $\frac{1}{K}$ \cite{atia2012boolean} when $K$ is known. However, for $K=N^{0.3}\approx 8$, this number is large (S=125), and using larger stimulation groups (S$>$20) requires many more iterations of Adam to converge as well as a smaller learning rate (e.g. 200-400 iterations and step size of $\sim$ 0.001), making it impractical. Experimentally, it may also make sense to limit the stimulation group size to avoid heating due to repeated photostimulation across large brain areas. 

\begin{figure}[!h]
    \includegraphics[width=\textwidth]{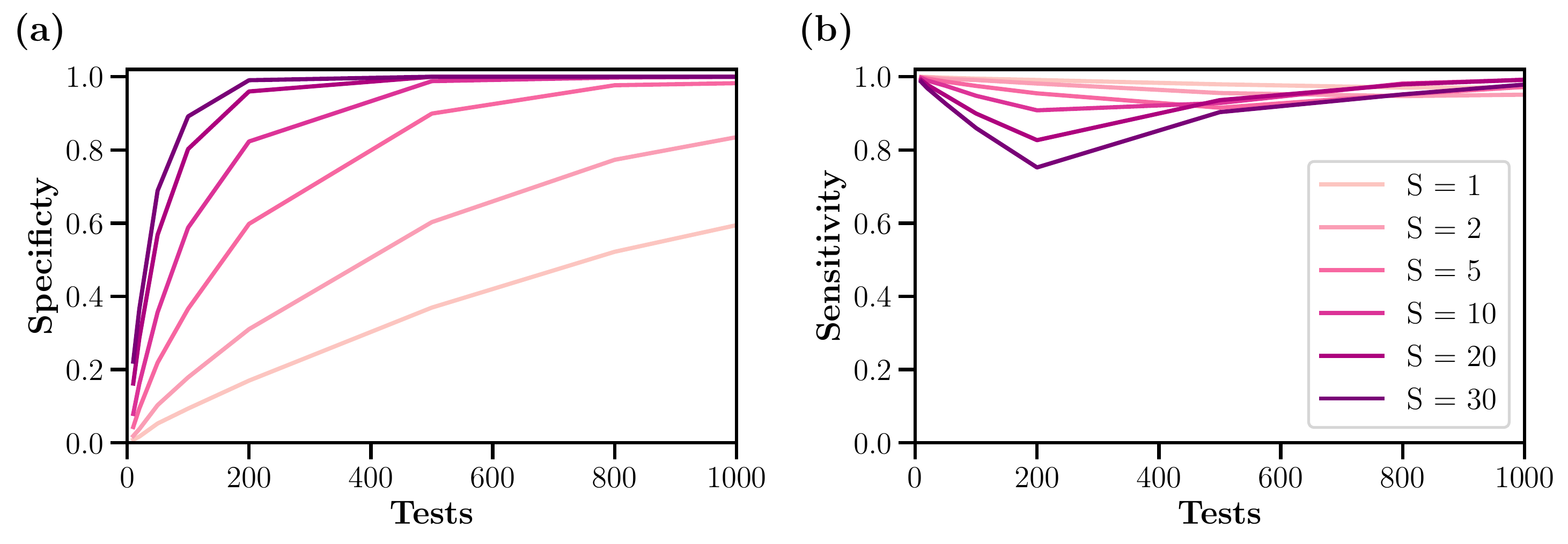}
    \linespread{1.0}\selectfont{\caption{\textbf{Number of neurons stimulated per test. (a, b)} Specificity and sensitivity, respectively, for different sizes of stimulation groups. 
    }
    \label{fig:suppfigNS}}
\end{figure}

\subsection*{Network recovery in online and adaptive settings}
Motivated by real-time, online experimental approaches that seek to intervene in live neural circuits with photostimulation \cite{emiliani2015all, packer2015simultaneous}, we also consider the online case. Here we use gradient descent (not Adam) and a sliding window of 1-10 tests to limit memory requirements and increase speed. Even with only a few fast gradient steps for each new test, we recover the network with the same level of sensitivity and specificity as in the batch case (Fig.~\ref{fig:fig3}).
This enables us to scale inference to much larger populations, even up to $N=10^4$ (Fig.~\ref{fig:fig3}) with an average processing time of $<$ 2s per stimulation, for an estimated experiment time of $\sim$ 1.5 hours for 2500 tests. 

\begin{figure}[!h]
    \includegraphics[width=\textwidth]{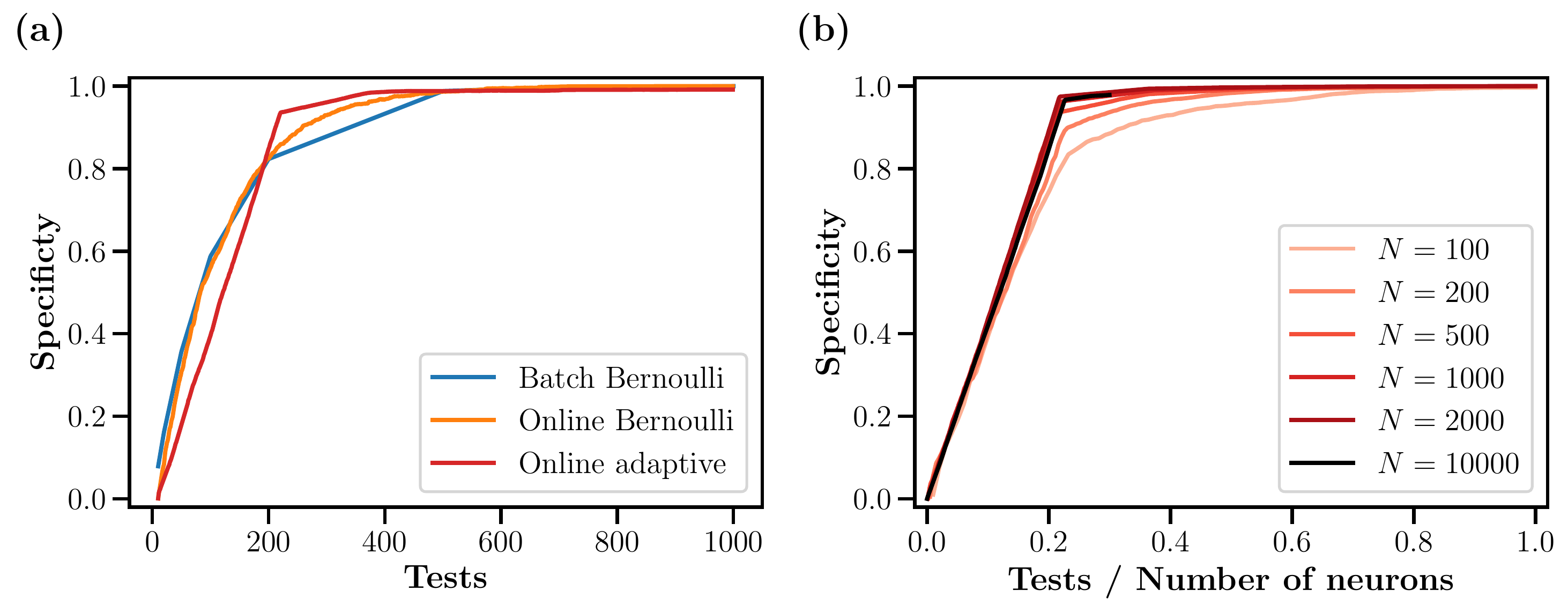}
    \linespread{1.0}\selectfont{\caption{\textbf{Online and adaptive network inference. (a)} Specificity as a function of the number of tests for the naive, online Bernoulli, and online adaptive designs. Performance is similar to the batch case, with the online adaptive approach requiring the fewest tests overall. {\bf (b)} Specificity as a function of the scaled number of tests $T$ (normalized $N$) for different system sizes in the adaptive case. The adaptive case exhibits an inflection point that moves toward $T\approx 0.2 N$ for large $N$. 
    }
    \label{fig:fig3}}
\end{figure}

\subsection*{Uncertainty in test error rates}
In our model (\ref{defn:stat_model}), we have assumed that the true and false positive rates for our test $h$, $\alpha$ and $\beta$, are known accurately. And for many tests of interest, these two quantities may be known \emph{theoretically}, provided the supplied data match the assumptions of the test. But when applied to real biological data assumptions are likely to be violated, and consequently, we may not know $\alpha$ and $\beta$ precisely. Here, we show both empirically and theoretically how this model misspecification affects our results.

Emprirically we observed essentially no difference if we misspecify the error rates; either if we assume them to be lower than they are or if we assume assume them to be higher than they are. The rates that matter are those of the test itself. Figure \ref{fig:suppfigmisspec} shows a case where the assumed $\alpha$ and $\beta$ are highly disparate (0.0001 and 0.45, respectively) and the true $\alpha$ and $\beta$ are those of the base case, 0.05, as well as less disparate but still misspecified cases (e.g. $\alpha=0.1, \beta=0.01$). The model consistently shows a negligible difference from the fit achieved from base case, where we use the true $\alpha$ and $\beta$.

\begin{figure}[!h]
    \includegraphics[width=\textwidth]{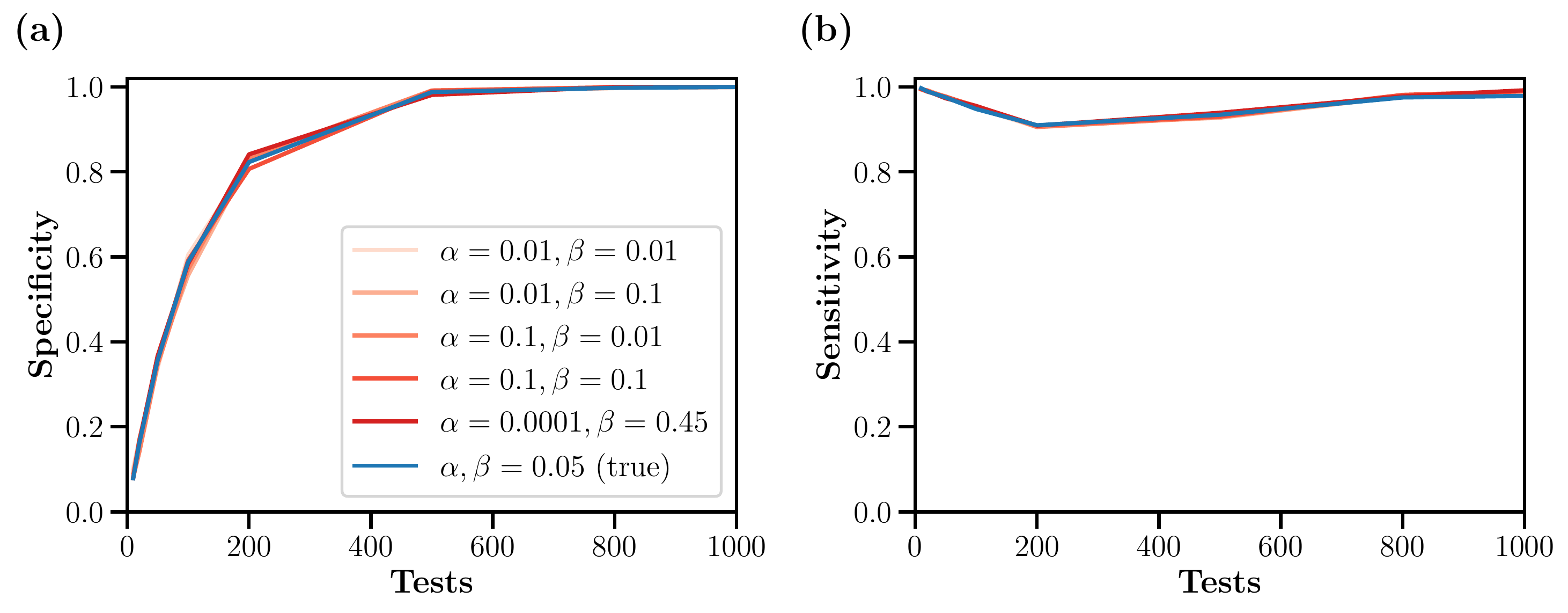}
    \linespread{1.0}\selectfont{\caption{\textbf{Model misspecification. (a, b)} Specificity and sensitivity, respectively, for the base case ($\alpha, \beta = 0.05$) in the main text (blue) and misspecified models with varying $\alpha$ and $\beta$ both over and under confident comapred to the true test error rates, including a disparate case with $\alpha = 0.0001$ and $\beta = 0.45$ (orange).
    }
    \label{fig:suppfigmisspec}}
\end{figure}

From Figure \ref{fig:suppfigmisspec}, it is apparent that model misspecification appears to have negligible impact on network recovery. Here, we show that this is in fact the case under very mild conditions. To do this, we begin with (3), in the $T \rightarrow \infty$ limit, so we can replace averages over trials with expectations over stimulation patterns:
\begin{equation}
    \label{bin_LL_misspec}
    \log p(\out | \conn, \stim) = T\left[\log\frac{(1-\alpha')(1-\beta')}{\alpha'\beta'} \mathbb{E}_\stim[\out\cdot \pred(\conn, \stim)] - \log\frac{1-\alpha'}{\beta'}\mathbb{E}_\stim[\pred(\conn, \stim)] + \text{const} \right] + o(T) \,
\end{equation}
where we have \emph{not} assumed that the error rates for the likelihood model ($\alpha'$, $\beta'$) are the same as those for the actual data-generating process ($\alpha$, $\beta$).

Fortunately, for the Bernoulli model, in which each neuron is stimulated i.i.d. with probability $p$, we can calculate the expectations in (\ref{bin_LL_misspec}). Let $A = \lbrace \stim \vert \pred(\conn, \stim) = 1\rbrace$, $Y = \lbrace \stim \vert \out = 1\rbrace$, and $\omega = \sum_i \conn_i$ is the number of nonzero connections. Computing the expectations and reinserting in (\ref{bin_LL_misspec}), we then have (dropping constants)
\begin{equation}
    \frac{1}{T}\log p(\out | \conn, \stim) \rightarrow \mathcal{L} = c_+(1-\beta-\alpha)\,p(A_* \cap A)  + (\alpha c_+ - c_-)\,p(A) \, ,
\end{equation}
where $c_\pm$ involve logarithms of $\alpha'$ and $\beta'$: $c_+ = \log \frac{1-\beta'}{\alpha'} + c_-$, and $c_- = \log \frac{1-\alpha'}{\beta'}$. Here $*$ indicates quantities calculated in the true data-generating model.

Maximizing this likelihood gives $p(A \cap A_*) = 1 - (1-p)^{\min(\omega, \omega_*)}$ and
\begin{equation}
    \mathcal{L} = \begin{cases} 
        \left((1-\beta)c_+ - c_-\right)(1 - (1-p)^\omega) & \omega < \omega_* \\
        c_+(1-\beta - \alpha)p(A_*) - (c_- - \alpha c_+)(1 - (1-p)^\omega) & \omega > \omega_*
    \end{cases} \nonumber \, .
\end{equation}
which have the same optimum solution, $\omega = \omega_*$, \emph{independently of $c_\pm$} provided
\begin{equation}
    \frac{1-\alpha}{\alpha} > \frac{\log \frac{1-\beta'}{\alpha'}}{\log \frac{1-\alpha'}{\beta'}} > \frac{\beta}{1 - \beta} \label{misspec_bound} \, .
\end{equation}
Of course, if $\alpha$, $\beta < 0.5$ and $\alpha' = \beta'$, this is \emph{always} satisfied. In this case,  
likelihood maximization remains consistent even for a misspecified model, and we do not need accurate estimates of our test error rates to recover the true set of connections.


\section*{Discussion}
Group testing itself comprises a large literature, reviewed in \cite{du2000combinatorial} and more recently \cite{aldridge2019group}. The link between noisy group testing and information theory was established in \cite{malyutov1978separating, atia2012boolean, aldridge2012adaptive, baldassini2013capacity, chan2014non} for the noise models of false positives and dilutions and in \cite{sejdinovic2010note} for both false positives and negatives. These studies established asymptotically optimal numbers of tests maximum likelihood decoding. Linear programming relaxation as a means of efficiently solving the decoding problem was previously proposed in \cite{malioutov2012boolean}, where the objective was to identify the minimal set of positives under an arbitrary noise model. Our approach differs in relaxing both the Boolean sums $\pred_t$ and the defects $\conn_i$, as well as assuming a more specific noise model, which allows us to establish a novel connection between the solution of the relaxed convex program and Bayesian inference (\ref{relaxed_LP}).

In neuroscience, much previous work has focused on inferring functional connectivity from correlational data, either spike trains or calcium fluorescence imaging \cite{paninski2004maximum, okatan2005, pillow2008spatio, stevenson2008bayesian, vidne2012modeling,song2013identification, lutcke2013inference, pernice2013reconstruction, fletcher2014scalable, soudry2015efficient, zaytsev2015reconstruction, karbasi2018learning, ladenbauer2019inferring, latimer2019inferring}. These methods typically rely on likelihood-based models and make moderate to strong parametric assumptions about the data generation process. This can result in inaccurate network recovery, even in simulation \cite{lutcke2013inference, das2020systematic}. Even more problematic is the difficulty of accounting for unobserved confounders \cite{soudry2015efficient}, which can also arise in our setup when non-recorded units mediate functional connections. 

Our work is similar in setup to \cite{aitchison2017model}, which also considered the possibilities inherent in selective stimulation of individual neurons. That work also employed a variational Bayes approach, positing a spike-and-slab prior on weights and an autoregressive generative model of ensuing calcium dynamics. Also of note is \cite{bertran2018active}, which considered optimal adaptive testing of single neurons to establish functional connections. More closely related are the approaches in \cite{hu2009reconstruction, fletcher2011neural}, which used a compressed sensing approach to network recovery. Those works did recover synaptic weights up to an overall normalization but did not consider either adaptive stimulation or the online inference setting. The latter problem was considered in \cite{shababo2013bayesian}, which focused on measurement of subthreshold responses in somewhat smaller systems. 

By contrast with many of these approaches, ours makes relatively few statistical assumptions. We do not posit a generative or parametric model, only the existence of some statistical test for a change (not necessarily excitatory) in neuronal activity following stimulation. Moreover, our approach affords approximate Bayesian inference (which could be extended to exact inference at the cost of additional constraint forces added to (\ref{dual_decomp})), does not require pretraining on existing data, and scales well to large neural populations, making it suitable for use in online settings. 

However, our approach does make key assumptions that might pose challenges for experimental application. First, as Figure \ref{fig:fig2}b shows, tests with poor statistical power require many more stimulations to reach correct inference, and below some threshold number of trials, this decrease in performance may be significant. 
Second, our approach ignores the relative strength of connections, as we focus on the structure of the unweighted network. This drastically reduces the number of parametric assumptions but would require a second round of more focused testing if these were quantities of interest. 
Nonetheless, our results suggest significant untapped potential in the application of adaptive experimental designs to large-scale neuroscience.

\section*{Conclusion}
We have proposed to apply noisy group testing to the problem of inferring functional connections in a neural network. We showed that a relaxation of the maximum likelihood inference problem for this setup is equivalent to Bayesian inference on the binarized network links, and that this problem can be solved efficiently for large populations in the online setting. To our knowledge, this is the first application of group testing to connectivity inference in neuroscience and the first proposal for truly scalable network inference.

\section*{Acknowledgments}
Research reported in this publication was supported by a NIH BRAIN Initiative Planning Grant (R34NS116738; EA and JP), a Ruth K. Broad Biomedical Research Foundation, Inc.\ Postdoctoral Fellowship Award (AD), and a Swartz Foundation Postdoctoral Fellowship for Theory in Neuroscience (AD). We would like to thank Maxim Nikitchenko and Robert Calderbank for useful discussions and previous anonymous reviewers for suggesting several improvements.

Part of this work was previously presented at the 34th Conference on Neural Information Processing Systems (NeurIPS, 2020) \cite{draelos2020online}.


%
%
%
\bibliography{bib}





\newpage

\setcounter{figure}{0}
\renewcommand{\thefigure}{S\arabic{figure}}

\section*{Supporting Information}
\section{Entropy gradient bounds}
\label{app_H_bounds}

Here, we prove the following bounds for the gradients of the entropy, the weakest (and most efficient) of which we make use of in Algorithm 1:
\begin{align}
    4\left|w_i - \frac{1}{2}\right| &\le \left|\log \frac{w_i}{1 - w_i}\right| \le |\nabla_{w_i} \ent| \\ 
    &\le \max\left(\left\lvert w_i^- - \log \frac{w_i}{1-w_i}\right\rvert, \left\lvert w_i^+ - \log \frac{w_i}{1-w_i}\right\rvert \right) \nonumber \\
    4\left|a_t - \left(1 - \epsilon_t(w)\right)\right| &\le \left|\log \frac{a_t}{1 - a_t} - \log\left(\frac{1}{\epsilon_t(w)} - 1 \right) \right|\le |\nabla_{a_t} \ent| \\
    &\le \max\left(\left\lvert a_t^- - \log \frac{a_t}{1-a_t}\right\rvert, \left\lvert a_t^+ - \log \frac{a_t}{1-a_t}\right\rvert \right)  \; .\nonumber
\end{align}
with $\epsilon_t \equiv \prod_i w_i^{\stim_{ti}}$ and $w_i^\pm$, $a_t^\pm$ constants that depend on the other entries in $a$ and $w$.
Note that it is these quantities, rather than the entropy itself, that are important for regularization, since overall entropy bounds may depend crucially on constants that do not affect the optimization that defines $w_i$ and $a_t$. Rather, it is the entropy gradients that define the regularization ``forces'' that result in estimates that are either weaker (lower bound) or stronger (upper bound) than the true entropy gradients and thus estimates of $w_i$ that are closer to or farther away from 0 and 1. Indeed, as we shall see, \emph{both} the upper and lower bounds above derive from upper bounds on the entropy itself.

\subsection{Strong convexity bound}
We start with the following Lemma:
\begin{lemma}
For any exponential family distribution $p(\mathbf{x})$ with only Boolean sufficient statistics, $\ent[p(\mathbf{x})]$ is $\sigma$-strongly concave for $\sigma \in (0, 4]$.
\end{lemma}
\begin{proof}
Let $T_i(\mathbf{x})$ be the sufficient statistics and $\nu_i$ their natural parameters, so that
\begin{equation}
    p(\mathbf{x}) = \frac{e^{\sum_i \nu_i T_i(\mathbf{x})}}{\mathcal{Z}} \;, \label{exp_fam_form}
\end{equation}
from which follows the well-known exponential family results
\begin{align}
    \frac{\partial}{\partial \nu_i} \log \mathcal{Z} &= \mathbb{E}T_i \\
    \frac{\partial^2}{\partial \nu_i\partial \nu_j} \log \mathcal{Z} &= \frac{\partial \mathbb{E}T_i}{\partial \nu_j} = J_{ij} = \mathbb{E}[T_i T_j] - \mathbb{E}T_i \mathbb{E}T_j = \mathrm{cov}(T_i, T_j) \; .
\end{align}
That is, the Hessian of the negative free energy is both the covariance matrix of the sufficient statistics and the Jacobian of the mapping from the natural parameters to the means. Likewise, for the derivatives of the entropy, 
\begin{align}
    \ent &= \mathbb{E}[- \log p(\mathbf{x})] = -\sum_i\nu_i \mathbb{E}T_i + \log \mathcal{Z} \\
    \frac{\partial}{\partial \mathbb{E}T_j} \ent &= -\nu_j -\sum_k \frac{\partial \nu_k}{\partial \mathbb{E}T_j} \mathbb{E}T_k + \sum_k \mathbb{E}T_k\frac{\partial \nu_k}{\partial \mathbb{E}T_j} = -\nu_j \label{grad_H}\\
    \frac{\partial^2}{\partial \mathbb{E}T_i\partial \mathbb{E}T_j} \ent &= -\frac{\partial \nu_j}{\partial \mathbb{E}T_i} = -J^{-1}_{ij} \;,
\end{align}
which is really another way of saying that $\ent$ and $-\log \mathcal{Z}$ are convex duals, and is related to the Cram\'{e}r-Rao Bound.

Now, recall that for any binary variable $T$, we have $\mathrm{var}(T) \le \frac{1}{4}$, so the maximum eigenvalue of $\mathrm{cov}(T_i, T_j)$, which are all binary, is also $\frac{1}{4}$. From this, it follows that the minimum eigenvalue of $-\nabla^2 \ent$, which is its inverse, is at least 4.

Finally, recall that a continuously differentiable convex function $f(\mathbf{x})$ is $\sigma$-strongly convex for some $\sigma > 0$ if we have, for all $\mathbf{y}$ in $\mathrm{dom}(f)$, 
\begin{equation}
   f(\mathbf{x}) \ge f(\mathbf{y}) + \nabla f(\mathbf{y}) \cdot (\mathbf{x} - \mathbf{y}) + \frac{\sigma}{2} \lVert \mathbf{x} - \mathbf{y} \rVert^2 \; ,
\end{equation}
which is equivalent to $\nabla^2 f \succeq \sigma \mathbb{I}$ [1]. Clearly, this is true when $\sigma$ is no larger than the minimum eigenvalue of $\nabla^2 f$, and we have that $-\ent$ is strongly convex for $\sigma \le 4$.
\end{proof}

In our case, we take $\mathbf{x} = \conn$, $T = (\conn, \pred(\conn))$ and $\nu = (\nu, \gamma)$. Our plan is to expand this around the maximum of $\ent$. This point is achieved at $\nu_i = \gamma_t = 0$ and corresponds to independent $\conn_i$ with $w_i = \frac{1}{2}$ and $a_t = 1- \epsilon_t(0) = 1 - \left(\frac{1}{2}\right)^{\sum_i \stim_{ti}}  \approx 1$ when the number of units tested is large. Then, from the lemma and the definition of strong convexity,
\begin{equation}
    \ent \le \ent_{sc} = N \log 2 - 2 \left\lVert \mathbf{w} - \frac{1}{2} \right\rVert^2 - 2 \left\lVert \mathbf{a} - 1 + \boldsymbol{\epsilon} \right\rVert^2 \; .
\end{equation}
Finally, since we have $\ent = \ent_{sc}$ and $\nabla \ent = \nabla \ent_{sc} = \mathbf{0}$ at $w_i = a_t = 0$, and $-\nabla^2 \ent \succeq -\nabla^2 \ent_{sc}$ from above, we have $|\nabla \ent_{sc}| \le |\nabla \ent|$ everywhere.

\subsection{Independent connections bound}
The second, stronger lower bound can be derived by once again considering the exponential family form (\ref{exp_fam_form}). For binary variables, we can write
\begin{equation}
    \mathbb{E}T_i \propto e^{\nu_i}\sum_{\mathbf{x}} T_i(\mathbf{x}) e^{\sum_{j \neq i} \nu_j T_j(\mathbf{x})} \propto e^{\nu_i}\mathbb{E}_{-i}T_i \;, 
\end{equation}
which gives
\begin{equation}
    \nu_i = \log \frac{\mathbb{E}T_i}{1 - \mathbb{E}T_i} - \log \frac{\mathbb{E}_{-i}T_i}{1 - \mathbb{E}_{-i}T_i} \;. \label{grad_H_log_diff}
\end{equation}
From (\ref{grad_H}), this is $-\nabla\ent$. The first term on the right-hand side involves expectations we assume known, while the second involves expectations in a reduced model with $\nu_i = 0$. Thus, if we were able to calculate $\mathbb{E}_{-i}T_i$, we could calculate $\nabla \ent$ exactly. Unfortunately, this calculation is intractable in general. However, specializing to our case, if we consider $\ent$ as a function of $(\nu, \gamma)$, then concavity gives
\begin{equation}
   0 = |\nabla_i \ent (0, 0) | \le |\nabla_i \ent (\nu, 0) | \le |\nabla_i \ent (\nu, \gamma) | \; .
\end{equation}
The middle term, with $\gamma = 0$, corresponds to a model with independent $\conn_i$, where we can easily calculate all expectations in (\ref{grad_H_log_diff}), giving 
\begin{align}
    \left\lvert \log \frac{w_i}{1 - w_i} \right\rvert &\le |\nabla_{w_i} \ent | \label{w_lowers} \\
    \left\lvert \log \frac{a_t}{1 - a_t} - \log \frac{1 - \prod_i w_i^{\stim_{ti}}}{\prod_i w_i^{\stim_{ti}}} \right\rvert &\le |\nabla_{a_t} \ent | \label{a_lowers} \;.
\end{align}

\begin{figure}[!h]
\label{H_bounds_fig}
\centering
\includegraphics[width=1\textwidth]{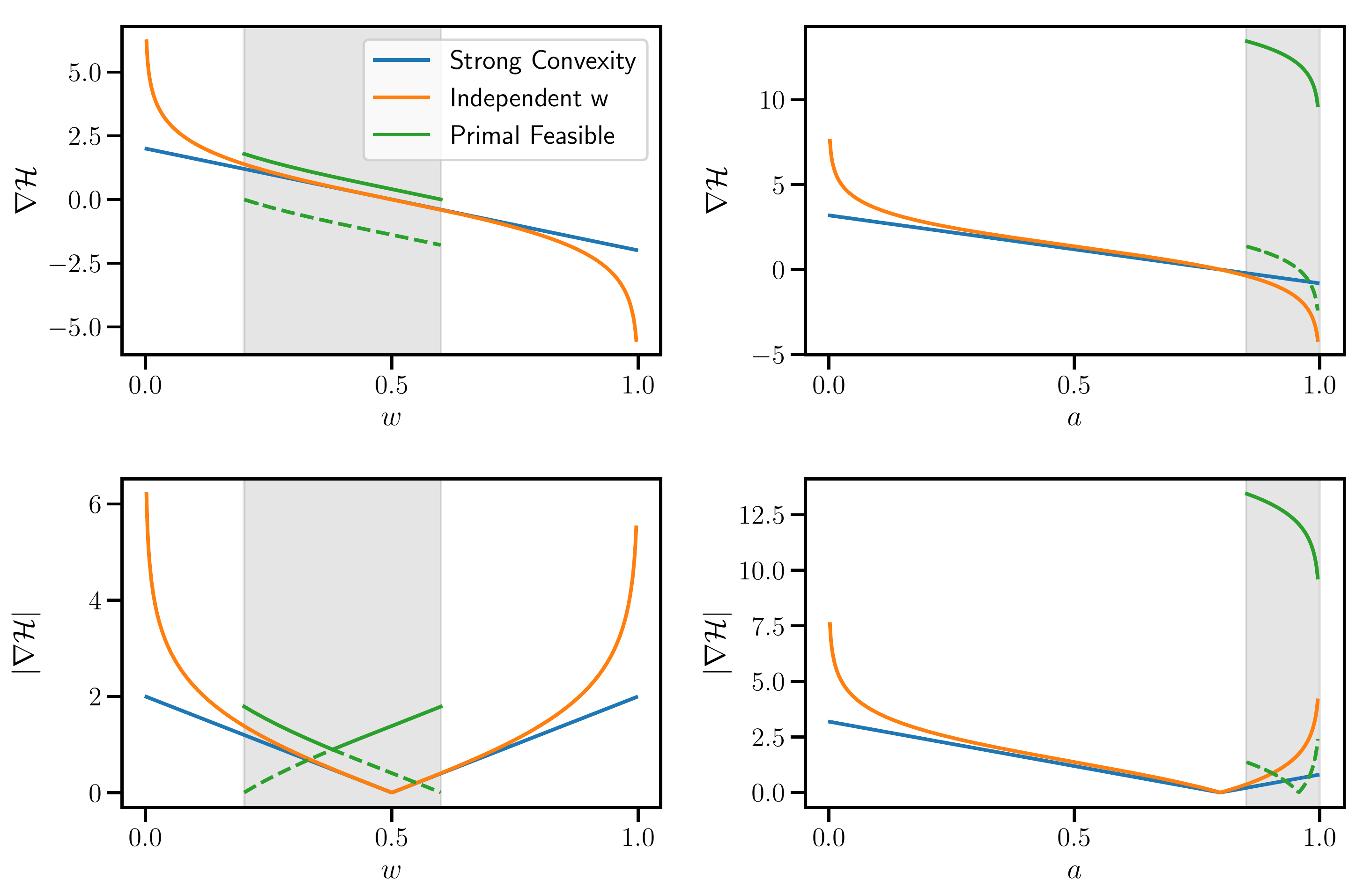}
\caption{{\bf Comparison of entropy gradient bounds.} Plots of $\nabla \ent$ (top) and its magnitude $|\nabla \ent|$ (bottom) for representative (unrelated) cases of $w$ (left) and $a$ (right). Magnitude lower bounds based on strong convexity and independent $\conn$ (\ref{w_lowers} - \ref{a_lowers}) are close near the maximum entropy point and diverge with distance from it. Upper (solid) and lower (dotted) bounds based on feasibility constraints (shaded region) for $(w, a)$ (\ref{w_feasible_bounds} - \ref{abs_grad_H_feasible}) likewise show an increasing gap near the endpoints of the interval. Importantly, the bounds for each $w_i$ depend on all $a_t$ in which it participates, and the bounds for $a_t$ depend on all $w_i$ tested. Lower bounds on $|\nabla \ent|$ produce less regularized, optimistic estimates of $w$ and $a$, while upper bounds produce conservative estimates biased toward the maximum entropy point.}
\end{figure}

\subsection{Feasibility bounds}
A final approach to bounding $|\nabla \ent|$ again starts from (\ref{grad_H_log_diff}), but this time simply bounds the second term based on mutual constraints among the parameters $w_i$ and $a_t$. That is, we again want to calculate $\mathbb{E}_{-i}T_i$, the mean of $T_i(\mathbf{x})$ under the exponential family distribution with no constraints on $T_i$ but all other sufficient statistic means specified. So, for example, we want $w_i$ calculated under the maximum entropy distribution with $(w_{i\neq j}, a_t)$ specified. Yet recall that the definition $\pred_t \equiv \max(\stim_{ti} \conn_i)$ implies constraints on $a_t = \mathbb{E}\pred_t$ and $w_i = \mathbb{E}\conn_i$:
\begin{equation}
    w_i \stim_{ti} \le a_t \le \sum_i \stim_{ti} w_i \; .
\end{equation}
But this allows us to conclude that, for any $i$, $t$,
\begin{align}
    \max(\lbrace a_t - \sum_{j\neq i} \stim_{tj} w_j \rbrace \cup \lbrace 0 \rbrace) &\le w_i \le \min(\lbrace a_t \vert \stim_{ti} = 1\rbrace) \label{w_feasible_bounds}\\
    \max(\lbrace \stim_{tj} w_j \rbrace ) &\le a_t \le \sum_i \stim_{ti} w_i \label{a_feasible_bounds}\,.
\end{align}
That is, if constraints dictate that $w_i \in [w_i^-, w_i^+]$,  we have from (\ref{grad_H_log_diff})
\begin{align}
   w_i^- - \log \frac{w_i}{1 - w_i} &\le \nabla \ent_{w_i} \le w_i^+ - \log \frac{w_i}{1 - w_i} \label{grad_H_feasible} \\
   \min \left(|\nabla \ent_{w_i}^-|, |\nabla \ent_{w_i}^+|\right) &\le |\nabla \ent_{w_i}| \le \max\left(|\nabla \ent_{w_i}^-|, |\nabla \ent_{w_i}^+|\right) \label{abs_grad_H_feasible}\,,
\end{align}
with exactly analogous formulas for $a_t$. Note that the $w_i^\pm$ depend on \emph{both} the other $w_j$ with which $w_i$ appears in tests \emph{and} the $a_t$ for the tests including it, while the $a_t^\pm$ depend only on those connections $w_i$ tested on trial $t$. Moreover, these latter bounds allow the constraints in (8) from the main text to be included in $w^\pm$ and $a^\pm$, which can be used (at some computational cost) to derive conservative bounds on the true posteriors by upper bounding $\nabla {\ent}$.

\subsection{Inference with binary entropy}
In contrast with our best recovery approximation, inference with $\ent_2(x) = -x \log x - (1-x) \log(1-x)$ (see (14) in main text) requires a much greater number of tests to reach the same level of specificity and sensitivity given a classification boundary of 0.5. Here, we present results for a smaller system, $N=200$ (Fig. \ref{fig:suppfigexact}). In general, this model exhibits many fewer false positives (specificity $\sim$ 1), while the posteriors for the true positive connections are less confident than the approximate case, ranging from 0.5 to 0.8 (when the approximate estimates are  $>0.8$). That is, overconfidence generally benefits recovery performance, while a decision rule based on the posterior marginals from tighter bounds requires many more tests for the same level of accuracy. This is at least in part due to the fact that the marginals fail to capture interactions among the $\conn$, and so are expected to underperform estimates like the true MAP, which do. 

\begin{figure}
    \includegraphics[width=\textwidth]{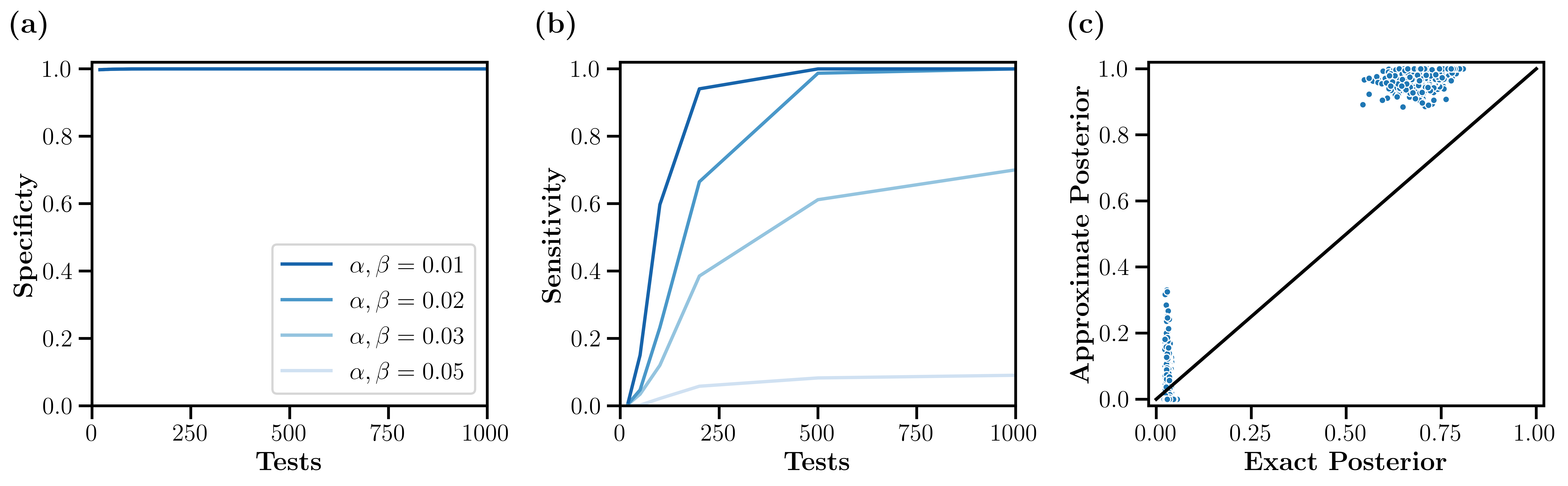}
    \linespread{1.0}\selectfont{\caption{\textbf{Recovery using binary entropy bounds. (a, b)} Specificity and sensitivity, respectively, for different test error rates.
    \textbf{(c)} Calibration plot comparing the weights obtained using the quadratic entropy bound and those obtained with the binary entropy bound (N=200, T=1000, $\alpha, \beta = 0.02$).
    }
    \label{fig:suppfigexact}}
\end{figure}

\section{Naive baseline model}
\label{app:baseline_model}
As a baseline model for network recovery, we consider two versions of a naive protocol based on individual cell ($S=1$) stimulations. For each test, a target neuron is randomly chosen (i.i.d.) from the entire population. In the first method, responses (0 or 1, according to the output of the hypothesis test) for each other neuron in the network are recorded, and these are used to update connection estimates based on a running mean. That is, the outgoing connections for the target neuron are updated each time the neuron is stimulated. All connections are initialized to zero, and connections that produce a result more than 50\% of the time are set to 1. This method was used for the naive comparison in the main text.

A second analysis approach for the same stimulation protocol is to use Bayesian inference, placing Beta priors on each connection that favor non-existence (e.g., $a = 1$, $b = 10$). In this case, recovery is based on a thresholded version of the maximum a posteriori estimate given $n_1$ responses and $n_0$ non-responses to stimulation: $\conn_{ij} =1$ if
\begin{equation}
    w_\mathrm{{MAP}} = \frac{a + n_1 - 1}{a + b + n_0 + n_1 - 2} > \frac{1}{2} \, .
\end{equation}
If $a, b = 1$ this reduces exactly to the first naive method. The stronger the bias towards 0 in the prior, the more tests are required to correctly infer the true connections, but the number of false positives is greatly reduced. 

Figure \ref{fig:suppfignaive} shows the results of all tested naive approaches. The first method of averaging used in the main text initialized all connections to zero (solid green line); here we also show the case where all connections are initialized to 0.5 (dotted green line), and the roles of specificity and sensitivity are effectively reversed. Finally, the second method using Bayesian inference (pink) with a Beta prior (a = 1, b = 5) requires many more tests to reach the same level of sensitivity, but is most successful at remaining highly specific, similar to the case of Bayesian inference with binary entropy bounds in our new approach (see section 1.4). 

\begin{figure}[!h]
    \includegraphics[width=\textwidth]{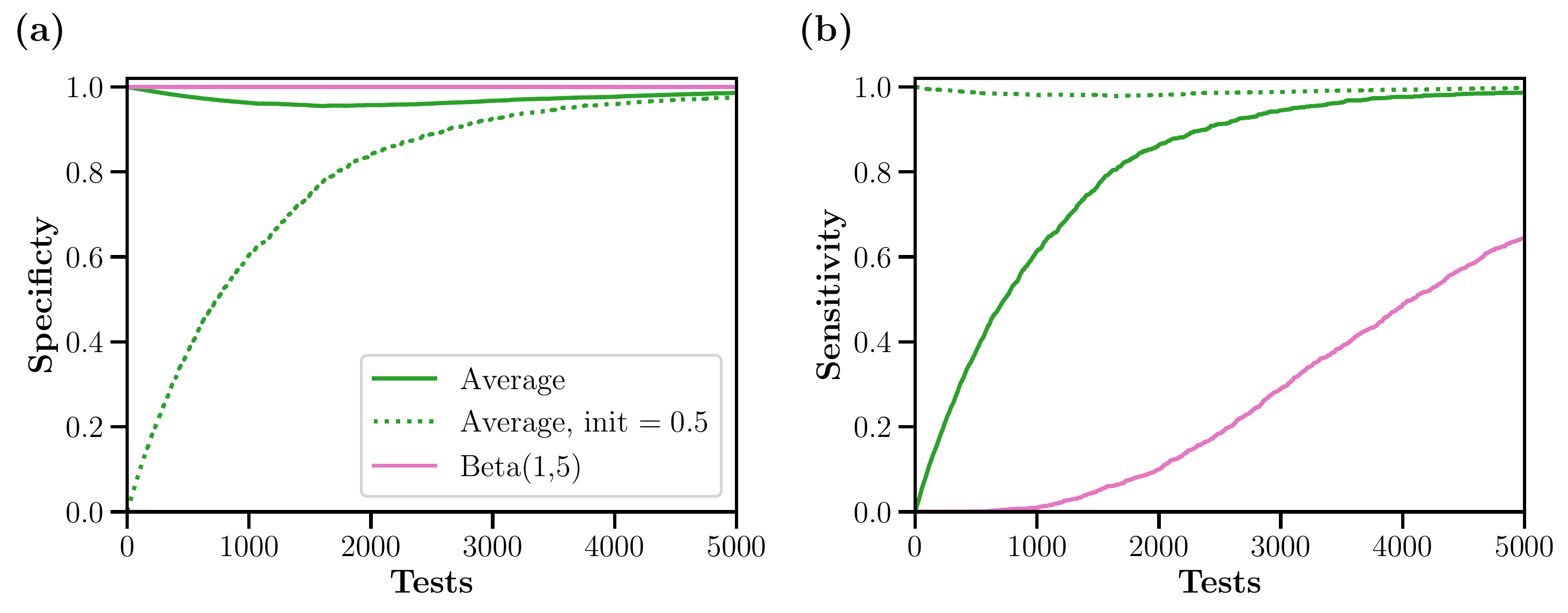}
    \linespread{1.0}\selectfont{\caption{\textbf{Naive methods. (a, b)} Specificity and sensitivity, respectively, for the naive methods tested. The average case (solid green) was used in the main text.
    }
    \label{fig:suppfignaive}}
\end{figure}

\section{Bayesian analysis of uncertain error rates}
In a full Bayesian analysis, we can consider placing priors on the test error rates:
\begin{align}
    \alpha &\sim \mathrm{Beta}(\phi_+, \phi_-) \\
    \beta &\sim \mathrm{Beta}(\varphi_+, \varphi_-)
\end{align}
Combining this with (2), we again have (3) from the main text, but we must now marginalize over our uncertainty in $\alpha$ and $\beta$. That is, we want
\begin{align}
    p(\out|\conn, \stim) &= \int p(\out|\conn, \stim, \alpha, \beta) p(\alpha) p(\beta) \, d\alpha\, d\beta \label{beta_marginal}\\
    &= \frac{B(\phi_+ + \nfp, \phi_- + \ntn)B(\varphi_+ + \nfn, \varphi_- + \ntp)}{B(\phi_+, \phi_-)B(\varphi_+, \varphi_-)}\, ,
\end{align}
where $B(x, y)$ is the beta function, $\ntp$ is the number of true positives ($\pred_t = 1$, $\out_t =1$), and similarly for the other expressions.
We would like to relate this quantity to (3). The easiest way to do this is to consider the limit of large numbers of tests, so that the beta functions are given by Stirling's approximation to $\Gamma(x)$. That is, 
\begin{equation}
    B(x, y) \sim \sqrt{2\pi}\frac{x^{x - \frac{1}{2}} y^{y - \frac{1}{2}}}{(x + y)^{x + y -\frac{1}{2}}} \, ,
\end{equation}
so that (\ref{beta_marginal}) gives
\begin{multline}
    \log p(\out|\conn, \stim) = \nfp \log \left(\frac{\phi_+ + \nfp}{\phi_+ + \phi_- + \nfp + \ntn} \right) + \ntn \log \left(\frac{\phi_- + \ntn}{\phi_+ + \phi_- + \nfp + \ntn} \right) \\ 
    + \nfn \log \left(\frac{\varphi_+ + \nfn}{\varphi_+ + \varphi_- + \nfn + \ntp} \right) + \ntp \log \left(\frac{\varphi_- + \nfn}{\varphi_+ + \varphi_- + \nfn + \ntp} \right) \\
    - \frac{1}{2}\log (\phi_+ + \phi_- + \nfp + \ntn) 
    - \frac{1}{2}\log (\varphi_+ + \varphi_- + \nfn + \ntp) + \text{constant} \, , \label{marginal_log_prob}
\end{multline}
which can be put into correspondence with (3) (up to subleading logarithmic terms in $n$) if we identify 
\begin{align}
    \bar{\alpha} &= \frac{\phi_+ + \nfp}{\phi_+ + \phi_- + \nfp + \ntn} \label{post_mean_alpha}\\
    \bar{\beta} &= \frac{\varphi_+ + \nfn}{\varphi_+ + \varphi_- + \nfn + \ntp} \label{post_mean_beta} \; . 
\end{align}
Of course (\ref{post_mean_alpha}) and (\ref{post_mean_beta}) are just the posterior means of $\alpha$ and $\beta$, and we see that in the limit of large numbers of tests, the logarithmic terms in $n$ can be ignored relative to the linear terms and the log evidence concentrates around the parameters of the data generating process. This in turn suggests an empirical Bayes approach in which we alternate variational inference (with $n$s fixed) with adjustment of the $n$s based on posterior estimates of the $\mathbb{E}\pred_t$. Fortunately, this alternation would only be necessary until the estimates of error rates stabilized, which can happen rapidly when we pool across the (assumed) independent sets of input connections. That is, for a population of $N$ neurons, one observes $N$ outcomes for each stimulation $t$, suggesting accurate estimation in only a small number of trials $T$ (provided the $\mathbb{E}\pred_t$ estimates are not changing rapidly). We leave this possibility for future work.

\end{document}